\documentclass{article}

\usepackage{microtype}
\usepackage{graphicx}
\usepackage{subcaption}
\usepackage{booktabs} 

\usepackage{hyperref}

\usepackage{makecell}
\usepackage[ruled,linesnumbered]{algorithm2e}

\usepackage[preprint]{icml2026}

\usepackage{amsmath}
\usepackage{amssymb}
\usepackage{mathtools}
\usepackage{amsthm}

\usepackage[capitalize,noabbrev]{cleveref}

\theoremstyle{plain}
\newtheorem{theorem}{Theorem}[section]

\theoremstyle{definition}
\newtheorem{definition}[theorem]{Definition}

\theoremstyle{remark}

\usepackage[textsize=tiny]{todonotes}

\usepackage{amsmath,amsfonts,bm}

\def\eqref#1{equation~\ref{#1}}

\def\1{\bm{1}}

\DeclareMathAlphabet{\mathsfit}{\encodingdefault}{\sfdefault}{m}{sl}
\SetMathAlphabet{\mathsfit}{bold}{\encodingdefault}{\sfdefault}{bx}{n}

\usepackage{hyperref}
\usepackage{url}

\usepackage{graphicx}
\usepackage{enumitem}

\icmltitlerunning{Learning Provably Correct Distributed Protocols Without Human Knowledge}

\begin{document}

\twocolumn[
  \icmltitle{Learning Provably Correct Distributed Protocols Without Human Knowledge}



  \icmlsetsymbol{equal}{*}

  \begin{icmlauthorlist}
    \icmlauthor{Yujie Hui}{osu}
    \icmlauthor{Xiaoyi Lu}{ufl}
    \icmlauthor{Andrew Perrault}{osu}
    \icmlauthor{Yang Wang}{osu}
  \end{icmlauthorlist}

  \icmlaffiliation{osu}{The Ohio State University}
  \icmlaffiliation{ufl}{University of Florida}

 \icmlcorrespondingauthor{Yang Wang}{wang.7564@osu.edu}

  \icmlkeywords{Machine Learning, ICML}

  \vskip 0.3in
]



\printAffiliationsAndNotice{}  

\begin{abstract}

Provably correct distributed protocols, which are a critical component of modern distributed systems, are highly challenging to design and have often required decades of human effort. These protocols allow multiple agents to coordinate to come to a common agreement in an environment with uncertainty and failures.
We formulate protocol design as a search problem over strategies in a game with imperfect information, and the desired correctness conditions are specified in Satisfiability Modulo Theories (SMT). However, standard methods for solving multi-agent games fail to learn correct protocols in this setting, even when the number of agents is small.
We propose a learning framework, GGMS, which integrates a specialized variant of Monte Carlo Tree Search with a transformer-based action encoder, a global depth-first search to break out of local minima, and repeated feedback from a model checker. Protocols output by GGMS are verified correct via exhaustive model checking for all executions within the bounded setting. We further prove that, under mild assumptions, the search process is complete: if a correct protocol exists, GGMS will eventually find it. In experiments, we show that GGMS can learn correct protocols for larger settings than existing methods.

\end{abstract}

\section{Introduction}
\label{sec:intro}

Consider a coordination game where multiple agents must reach agreement on a shared decision, but each agent observes only its own local state and the messages it receives—some of which may be lost. An adversary controls which messages fail to arrive. The agents win if they all reach the same valid decision; they lose if any agent decides differently or violates a safety constraint. Crucially, this is not a game where high expected reward suffices—we need \emph{guaranteed} correctness under worst-case adversarial play. A protocol that works 99.9\% of the time is useless; a single counterexample renders it unusable.

This formulation captures the essence of \emph{distributed protocol design}, a fundamental problem in building reliable systems. When you use a database, make a payment, or store a file in the cloud, distributed protocols ensure that multiple machines either all agree on what happened, or the operation safely aborts—even when machines crash or messages vanish. Designing these protocols has historically required decades of human ingenuity: 
The above formulation represents the consensus problem, a classic topic in distributed protocols, and
consensus alone has motivated over 40 years of research spanning Paxos~\citep{Lamport1998Parliament,lamport2001paxos,Moraru2013EPaxos,Ongaro2014Raft} and Byzantine fault tolerance~\citep{Lamport1982Byzantine,Castro1999PBFT,Kotla2007Zyzzyva,Giridharan2024Autobahn}.

\vspace{-.05in}
\paragraph{A concrete example.} In the \emph{consensus} problem, three processes P1, P2, P3 must agree on a single value. Suppose P1 and P2 start proposing ``0'' while P3 proposes ``1''. In round one, each process broadcasts its proposal—but P3 crashes mid-broadcast, so P1 receives P3's message while P2 does not. Now P1 sees \{0, 0, 1\} while P2 sees \{0, 0, ?\}. Despite this asymmetry, a correct protocol must lead both to the same final decision. The FloodSet algorithm~\citep{LynchBook} achieves this in $f+1$ rounds (where $f$ is the maximum failures): each process repeatedly broadcasts everything it has received, and in the final round decides on the minimum value. The key insight is that at least one round must be failure-free, allowing all processes to synchronize (see \S\ref{sec:back-distributed}).

\vspace{-.05in}
\paragraph{Can we automate the discovery of such protocols?} The game-theoretic structure—partial observability, adversarial uncertainty, hard safety constraints—suggests this problem might be amenable to the search-based learning that succeeded in games like Go and poker. We formulate protocol design as search over strategies in an imperfect-information game, where correctness properties are formally specified and exhaustively verified.

We pursue this through \emph{zero-knowledge synthesis}: given only a specification of what constitutes a correct outcome, can a learning system discover protocols without human-designed examples? This methodology serves two purposes. First, it reveals what structure is \emph{necessary} for correctness versus merely \emph{conventional}—when our system independently discovers FloodSet-like protocols, this confirms the structure is dictated by the problem constraints rather than historical accident. Second, it provides a rigorous baseline: success validates that the search space is tractable, while the framework can naturally incorporate human knowledge (warm-starting, architectural constraints) when desired.

However, standard game-playing approaches fail here. In AlphaGo-style MCTS~\citep{silver2017mastering}, self-play learns policies that win \emph{in expectation}. Distributed protocols require something stronger: a policy that \emph{never} loses against \emph{any} failure pattern. Additionally, multiple correct protocols often exist, and naive learning can mix transitions from different protocols—a \emph{superposition problem}—causing failures even when each protocol would individually be correct. These challenges defeat prior methods even for 3–4 agents.

We develop \textbf{Guided Global Monte Carlo Tree Search (GGMS)}, integrating three key ideas. (1) \emph{Model checking as a hard oracle}: any candidate protocol is exhaustively verified against all possible executions (for the given process count). Violations produce counterexamples that feed back into training. This exhaustive verification—not the learning process—is what makes output protocols \emph{provably correct}. (2) \emph{Global depth-first search}: when learning produces ambiguous transitions (multiple outputs with similar probability), often due to the superposition problem mentioned above, GGMS freezes one choice and continues. If no correct protocol exists under current freezes, it backtracks systematically. This guarantees eventual convergence if a correct protocol exists. (3) \emph{Guided sampling}: a phased curriculum starts with less ambiguous scenarios (failures only in later rounds, unambiguous initial states), allowing the effects of frozen transitions to propagate before relaxing to harder cases.

\noindent We make the following contributions. (1) We formulate distributed protocol design as search over state machines in an imperfect-information game with formal correctness specifications and exhaustive verification. (2) We propose GGMS, combining MCTS with a transformer encoder, global DFS, and iterative model-checking feedback. We prove that under mild assumptions, the search is complete: GGMS will not miss correct solutions (\S\ref{sec:dfs}). (3) We demonstrate that GGMS learns correct protocols where standard MCTS fails, scaling to 4 processes with 3 failures. For consensus, it independently recovers FloodSet-like protocols. For synchronous atomic commit—a variant without known solutions—it discovers a novel correct protocol (\S\ref{sec:eval}).

\vspace{-.05in}
\paragraph{Relationship to LLMs.} We tested GPT and Gemini (Appendix~\ref{sec:GPT}): both retrieved FloodSet for consensus but often failed on synchronous atomic commit on the first attempt, since it lacks existing solutions. LLMs struggle with systematic exploration for novel synthesis, making GGMS complementary.

Our evaluation shows GGMS achieves substantially higher success rates than MCTS baselines across all tested configurations. We discuss limitations including synchronous network assumptions and bounded process counts, and outline directions for relaxing these assumptions and integrating with LLM-based approaches.

\section{Background}
\label{sec:background}

\subsection{Distributed Protocols as State Machines}
\label{sec:back-distributed}

The distributed systems community models protocols using the \emph{state machine approach}: each process is a state machine that takes its current state and incoming messages as input, and outputs a new state and messages to send. This enables both precise specification and formal verification.

Protocols vary along several dimensions. \emph{Failure models} specify what can go wrong: in \emph{crash failures}, a failed process simply stops responding; in \emph{Byzantine failures}, a failed process may behave arbitrarily. \emph{Timing models} specify synchrony assumptions: in \emph{synchronous} networks, message delays and clock drift are bounded; in \emph{asynchronous} networks, no such bounds exist. This paper focuses on crash failures in synchronous networks---the simplest non-trivial setting---and discusses extensions in \S\ref{sec:future}.

\vspace{-.05in}
\paragraph{The FloodSet algorithm.}
We formalize the consensus example from \S\ref{sec:intro}. Recall the requirements: (1) every correct process eventually decides, (2) all decisions are identical, and (3) the decision must be some process's initial proposal.

FloodSet~\citep{LynchBook} works in $f+1$ rounds, where $f$ is the maximum number of failures. Each process $p$ maintains a set $W$, initialized with $p$'s own proposal. Each round, every process broadcasts $W$ and updates $W := W \cup \bigcup_j \textit{Received}_j$. After $f+1$ rounds, each process decides $\min(W)$.

\emph{Why does this work?} Among $f+1$ rounds, at least one round has no failures (since at most $f$ processes can fail total). In that round, all surviving processes receive identical messages and reach the same $W$. This $W$ will not change in subsequent rounds, so $\min(W)$ remains consistent.

As a state machine: the state is $W$, the transition function unions received sets into $W$, and the final output applies $\min(\cdot)$.

\subsection{Model Checking}
\label{sec:back-model-checking}

Given a state machine and formal correctness properties, \emph{model checking} verifies that no execution violates the properties. For a fixed number of processes $N$, this can be done by exhaustive enumeration of all possible initial states and failure patterns. More efficient approaches use SMT solvers~\citep{de2008z3} or exploit symmetry~\citep{Leesatapornwongsa2014SAMC}.

For example, the consensus agreement property is formalized as:
\begin{multline}
\textbf{P1: } \neg \exists\, n, m \in N,\; f_n \in F_n,\; f_m \in F_m :\\ \big(f_n = \texttt{decision:0} \land f_m = \texttt{decision:1}\big)
\end{multline}
where $F_n$ denotes the final decisions of process $n$. Full property specifications appear in Appendix~\ref{sec:formal-properties}.

Model checking provides correctness guarantees for a \emph{concrete} $N$. Proving correctness for \emph{arbitrary} $N$ requires inductive formal verification~\citep{hawblitzel2015ironfleet,ma2019i4,yao2021distai,zhang2025basilisk}---an important direction we leave to future work.

\section{Problem Formulation}
\label{sec:modeling}

We now formalize the synthesis problem. Our goal is to \emph{learn} a state machine that satisfies correctness properties under all possible failure scenarios, given only a specification of what correct behavior means---not how to achieve it.

\subsection{State Machine Model}

Each process runs an identical deterministic state machine. The input at each round consists of: the current state, messages received from all processes (including itself), the round number, and optionally the process ID. The output is a new state, which is broadcast in the next round.

States fall into four categories:
\begin{itemize}[leftmargin=*, itemsep=2pt]
    \item \textbf{Initial states} $I$: possible starting states (e.g., \texttt{init:0}, \texttt{init:1} for binary consensus)
    \item \textbf{Decision states} $D$: final outputs visible externally (e.g., \texttt{decision:0}, \texttt{decision:1})
    \item \textbf{Lost state} $L$: a special symbol indicating a message was not received
    \item \textbf{Internal states}: intermediate states for coordination (e.g., \texttt{internal:a}, \texttt{internal:b})
\end{itemize}

The human specifies only $I$, $D$, and correctness properties $C$. The learning system must discover how many internal states are needed and what transitions to use---analogous to AlphaGo Zero learning game strategies given only the rules.

\subsection{Execution Model}

\begin{figure}[t]
    \centering
        \includegraphics[width=0.48\textwidth]{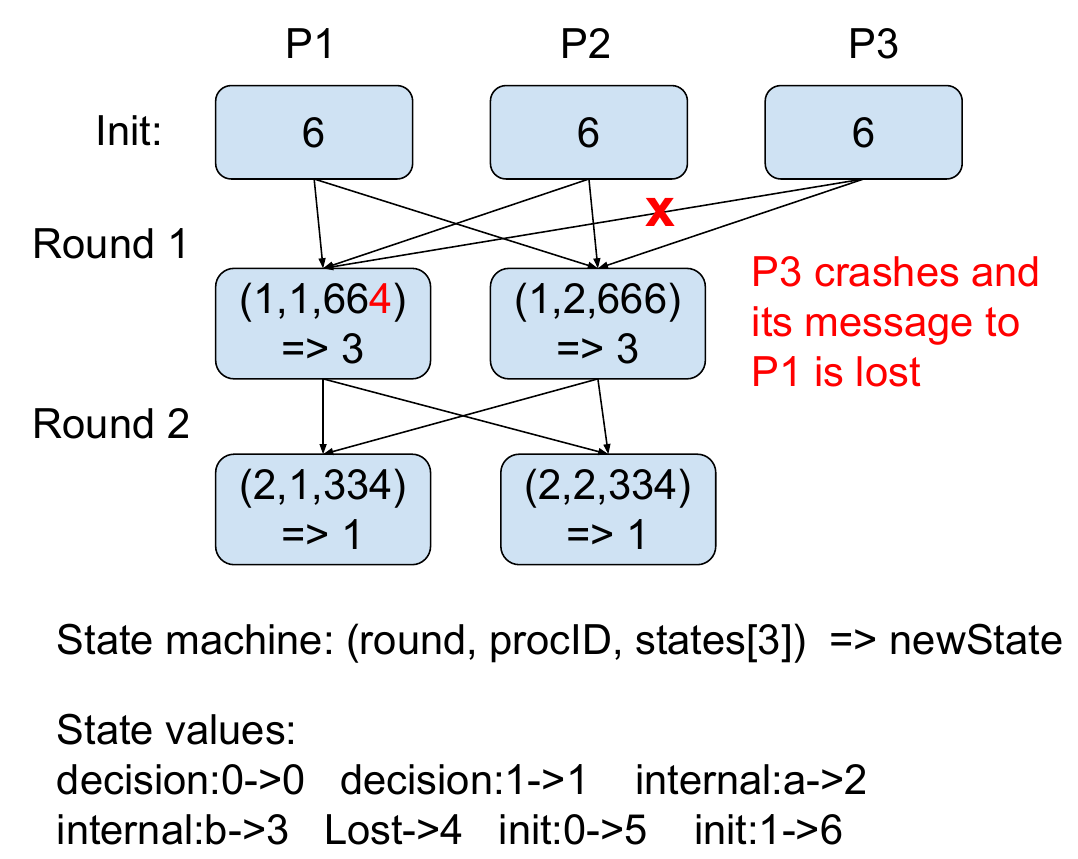}
        \caption{Simulating the FloodSet protocol. P3 crashes in Round 1 and causes P1 and P2 to have divergent inputs, but P1 and P2 still converge eventually following the protocol.}
        \label{fig:simulation_example}
        \vspace{-.1in}
\end{figure}

Execution proceeds in synchronous rounds. Before round 1, each process receives an initial state. In each round: (1) every process broadcasts its current state to all others, (2) some messages may be lost due to crashes, and (3) each process applies its transition function to compute a new state. In the final round, processes transition to decision states.

\paragraph{Modeling crashes.} We model crash failures through message loss: if process $p$ crashes during round $i$, some recipients may receive $p$'s round-$i$ message while others receive $L$ (lost). After crashing, all of $p$'s subsequent messages are lost. This captures the partial-broadcast semantics that make consensus non-trivial.

 Figure~\ref{fig:simulation_example} illustrates an execution of the FloodSet algorithm, where P3 crashes mid-broadcast in round 1, causing P1 and P2 to observe different inputs---yet both still converge to the same decision.

\subsection{Formal Problem Definition}

\begin{definition}[Protocol Specification]
A target protocol is a tuple $\mathit{DP} = (I, D, C)$: initial states $I$, decision states $D$, and correctness properties $C$.
\end{definition}

\begin{definition}[Setting]
A setting is a tuple $S = (n, r, f, k)$: number of processes $n$, rounds $r$, maximum failures $f$, and internal states $k$.
\end{definition}

\begin{definition}[Scenario]
A scenario $\mathit{Sc} = (\mathit{Init}, \mathit{Loss})$ specifies initial states $\mathit{Init}[1..n]$ and message losses $\mathit{Loss} \subseteq \{(i, j, t) : \text{message from } i \text{ to } j \text{ lost in round } t\}$.
\end{definition}

\begin{definition}[State Machine]
A state machine $\mathit{SM}$ is a transition function: $[\mathit{round}, \mathit{procID}, \mathit{inputStates}] \mapsto \mathit{newState}$.
\end{definition}

\noindent\textbf{Synthesis goal:} Given $\mathit{DP}$ and $S$, find $\mathit{SM}$ such that for \emph{all} scenarios $\mathit{Sc}$ consistent with $S$, executing $\mathit{SM}$ satisfies $C$.

A setting is \textbf{feasible} if such an $\mathit{SM}$ exists. Since feasibility is unknown a priori, our approach starts with small settings and increases $r$ or $k$ if synthesis fails.

\subsection{Assumptions and Limitations}

Our model makes simplifying assumptions:
\begin{itemize}[leftmargin=*, itemsep=2pt]
    \item \emph{Identical state machines}: all processes run the same code, precluding Byzantine behavior
    \item \emph{Broadcast only}: processes send identical messages to all others, precluding point-to-point protocols
    \item \emph{Synchronous timing}: messages sent in round $i$ arrive by round $i$'s end
    \item \emph{Final-round decisions}: processes decide only in the last round
\end{itemize}

These restrictions define a tractable starting point. We discuss relaxations in \S\ref{sec:future}.

\section{Learning a Distributed Protocol}
\label{sec:training}

\subsection{Monte-Carlo Tree Search}
\label{subsec:mcts}

\begin{figure}[t]
    \centering
        \includegraphics[width=0.48\textwidth]{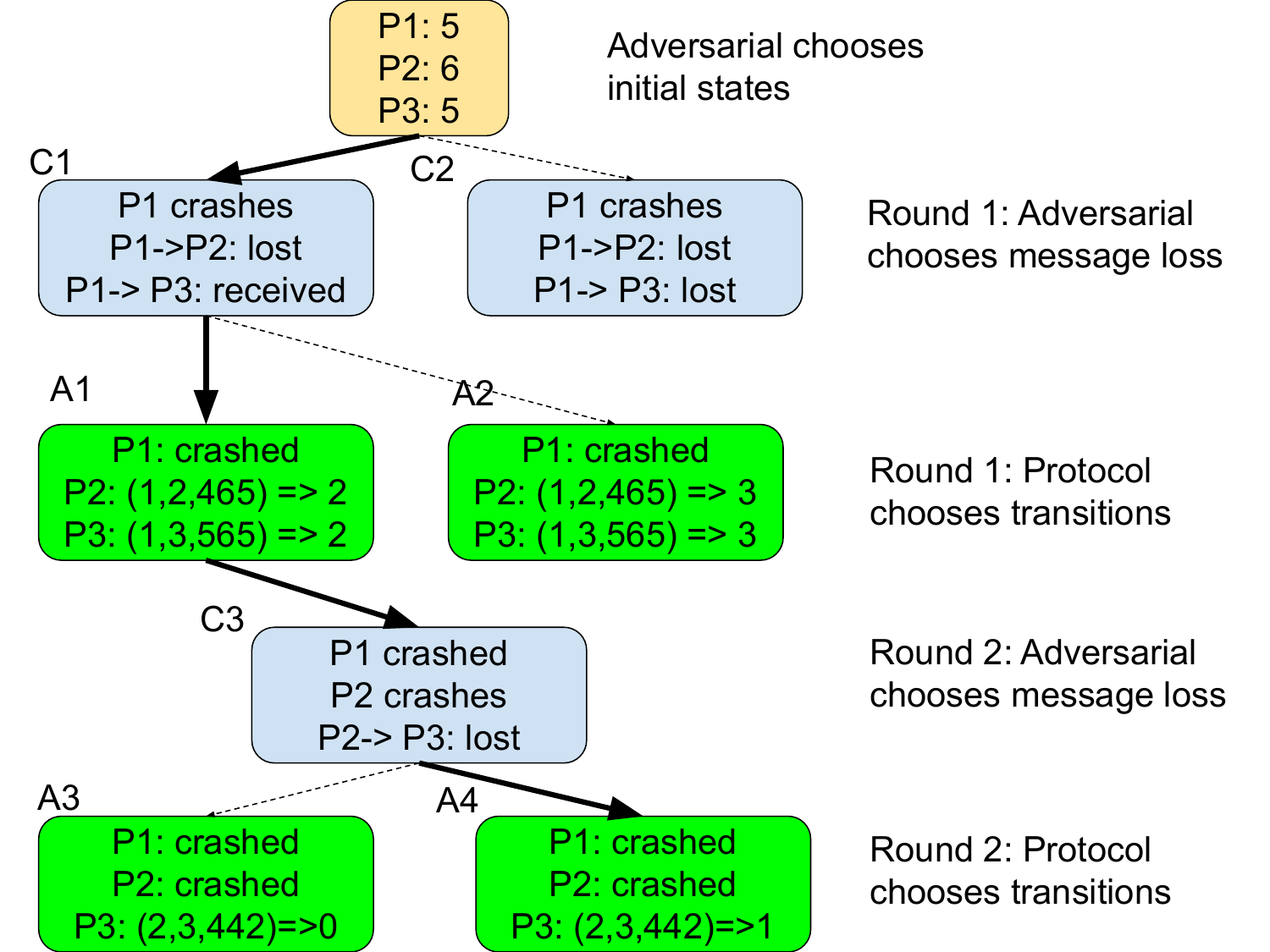}
        \caption{Applying MCTS. Adversarial chooses initial states and lost messages that are likely to violate correctness. Protocol chooses transitions that are likely to achieve correctness.}
        \label{fig:mcts}
         \vspace{-.1in}
\end{figure}

Inspired by the similarity between distributed protocols and board games, we first apply the Monte Carlo Tree Search (MCTS) simulation approach.

We model the whole process as two players in a game. The protocol player tries to find the right transitions in its state machine so that all processes in the distributed protocol
can always achieve the correctness properties. The adversarial player tries to
find the scenarios that can defeat the protocol
player's state machine, i.e. making it violate the correctness properties. While it
may be possible to train a model for the adversarial player as well, our current implementation
relies on random exploration for the adversarial player and relies on the model checker
to make the final checking.

The whole simulation process works as shown in Figure~\ref{fig:mcts}. 
In one simulation, the adversarial player first selects the initial state for each process and messages to lose
for the first round; our simulation follows the procedure in \S\ref{sec:modeling} to let each process
broadcast its state, apply the message loss selected by the adversarial player,
and compute the input for each process; then the protocol player selects the new state for each process
based on its input. The simulation repeats this process until the maximum number of rounds is reached.
By executing the simulation multiple times, allowing both players to explore different transitions, we can 
merge their results to build a search tree (Figure~\ref{fig:mcts}).

Our MCTS implementation is similar to that of AlphaGo-Zero, with the following differences. First,
we do not use the value network in our MCTS. In AlphaGo-Zero, the value network is used to predict the expected reward, which guides the MCTS simulations. However, the number of rounds of the distributed protocol is much smaller than that in Go. This means we can simulate until the end of a protocol to get the real reward without relying on the value network.

Second, in AlphaGo-Zero, the MCTS algorithm is also used for inference, i.e., when actually playing the game or running the protocol. However, we can only use the trained policy network for inference because the internal information, such as visited count, rewards, and probabilities, is invisible between different processes during actual running. In other words, during inference, among multiple transitions from the same input, a process will always apply the
one with the highest probability.

The details of our MCTS implementation are in \S\ref{sec:pseudo-mcts}.

\subsection{Ensuring Convergence with Global DFS}
\label{sec:dfs}
MCTS alone can occasionally converge to a correct state machine, but often fails
to do so. Our investigation shows that the primary reason is the superposition problem, that is, there often exist multiple versions of correct state machines, and MCTS may end up in a situation where it learns some transitions from one version
and some other transitions from another version, but when these transitions are combined, they do not generate
a correct state machine.

There are multiple reasons for the existence of multiple correct versions. First, 
the state machine has the flexibility to give a specific meaning to an arbitrary internal state, creating multiple
equivalent protocols. For example, for consensus, one version could use internal state
\texttt{internal:a} to represent the intention for \texttt{decision:0}, and another version could use \texttt{internal:b}
for the same purpose.

Second, some protocols inherently have some flexibility to allow different decisions. 
For example, for consensus, if some processes have \texttt{init:0} and some have \texttt{init:1},
they could either all decide \texttt{decision:0} or all decide \texttt{decision:1}.
Suppose that in such a scenario, a process has input A and
another process has input B at the same round. In a correct protocol, both A and B should
transit to the same state, either $\texttt{intent-0}$ or $\texttt{intent-1}$.
When simulating this scenario, MCTS can find this constraint and maybe give
$A \rightarrow \texttt{intent-0}$ and $B \rightarrow \texttt{intent-0}$ a higher
probability than $A \rightarrow \texttt{intent-1}$ and $B \rightarrow \texttt{intent-1}$.
However, A or B can also appear alone in other scenarios. When combing probabilities across scenarios, MCTS 
may end up with $A \rightarrow \texttt{intent-0}$ having a higher probability than $A \rightarrow \texttt{intent-1}$ but $B \rightarrow \texttt{intent-1}$ having a higher probability than
$B \rightarrow \texttt{intent-0}$, which violates the constraint that A and B should transit to the same state. We show detailed examples of these two reasons 
and how they affect MCTS in \S\ref{sec:examples-convergence}.

Our approach, Guided Global Monte Carlo Tree Search (GGMS), addresses this issue by enhancing MCTS with
Depth-First Search (DFS).
The key idea is that, if multiple transitions from the same input have similar probability, GGMS should try to freeze it to one transition
(i.e., don't allow random exploration for this input during MCTS) and then keep training to see whether it can get a correct state machine. 
In the above example, by freezing the transition $A \rightarrow \texttt{intent-0}$, we hope that keeping training will motivate $B \rightarrow \texttt{intent-0}$.
Since there might be multiple such ambiguity points, GGMS may repeat freezing multiple times. And since freezing may be wrong (e.g.,
freezing $A \rightarrow \texttt{intent-0}$ and $B \rightarrow \texttt{intent-1}$ for whatever reason), GGMS also needs to unfreeze certain transitions when it hits a dead end.
This procedure leads to a DFS-like search, in which GGMS keeps freezing certain transitions until either it succeeds in finding a correct
state machine or it hits a dead end. In the latter case, it unfreezes prior frozen transitions in an DFS manner.

\begin{theorem}[Search Completeness]
With a feasible setting, assuming GGMS's unfreezing condition is accurate, GGMS can eventually find a correct state machine.
\end{theorem}

This is because, for a specific setting, the number of possible state machines is finite.
Therefore, DFS can eventually explore all
possible state machines, ensuring that we can find a correct one.

However, naive DFS (i.e., randomly freezing a transition) may take
too long. Consider a protocol which involves two initial states,
two internal states, two decision states, three processes, and three rounds,
and assume that process ID does not affect
transitions, the total number of possible inputs to the state machine is
$3 (round) \times 2 (ownState) \times 3^2 (otherState) = 54$. And since
each of these inputs can transit to two values, the total number of state
machines is $2^{54}$. Using DFS to explore each state machine is too expensive.
By combining DFS and MCTS, GGMS relies on DFS to break ambiguity, and relies on MCTS to
provide hints about what transitions to freeze and ``propagate'' the effect of freezing (e.g., if we freeze $A \rightarrow \texttt{intent-0}$,
then MCTS can find that $B \rightarrow \texttt{intent-0}$).

We present the details of our DFS algorithm in \S\ref{sec:pseudo-dfs}, which includes the conditions
for freezing and unfreezing and how to determine which transitions to freeze or unfreeze. Our current unfreezing condition,
which is based on exhaustive search, is accurate but not scalable, and we discuss possible ways to replace exhaustive search
with a more scalable Z3-based solver.

Note that in practice, there is always a time limit for training, so despite
the eventual convergence guarantee, GGMS may not succeed within the time limit. However,
this property means that we can always devote more time and/or resources to
increase the chance of success.

\subsection{Accelerating Convergence with Guided MCTS}
\label{sec:guided}

While DFS helps to address the ambiguity issue, its speed is sometimes not satisfactory
due to the following reasons. First, as discussed before, since GGMS can give a meaning to an arbitrary
internal state, it needs multiple rounds
of freezing to break the ambiguity among them. With more rounds in the protocol and more states, 
this process requires more rounds of freezing. To address this problem, at the beginning of training,
GGMS freezes all transitions in one particular scenario (Sc), so that the simulation can reach
a correct decision. Such a group freezing strategy helps GGMS to break the ambiguity among internal states more
quickly. Furthermore, we can prove that, under certain conditions,
it will not hurt the capability of GGMS to find a correct state machine.

\begin{theorem}[Group Freezing Preserves Completeness]
    In a feasible setting, assuming that 1) the scenario of this particular simulation leads to a definite
    decision (i.e, no ambiguity),
    and 2) for each pair of (round, procID), this approach only freezes one transition
    (round, procID, inputA) $\rightarrow$ B, there exists a correct state machine compatible with all frozen transitions ($Transitions_{fix}$).
\end{theorem}

We provide the formal proof in \S\ref{sec:proof}. Intuitively, if a correct state machine has (round, procID, inputA) $\rightarrow$ C, we can always swap B and C for (round, procID) to get an equivalent protocol.

Second, as discussed before, we expect MCTS
to propagate the effect of freezing. In the above example, when GGMS freezes $A \rightarrow \texttt{intent-0}$,
it expects MCTS to find that it should choose $B \rightarrow \texttt{intent-0}$. In practice, such
propagation does not always succeed, since propagation often relies on a particular scenario
to identify the relationship. In the above example, in order for GGMS to learn that
A and B should transit to the same state, A and B should happen together in the same simulation.
In other scenarios where only B occurs, GGMS may find that it is OK for $B$
to transit to either \texttt{intent-1} or \texttt{intent-0}, due to the ambiguity problem discussed above. 
If simulations including both A and B happen rarely, but simulations
including only B happen often, GGMS may not give $B \rightarrow \texttt{intent-0}$ a high probability.
This problem is particularly troublesome when the random initialization of the state machine gives $B \rightarrow \texttt{intent-1}$
a high probability to begin with. 

To address this, we introduce a guided sampling method. After freezing the initial path, GGMS simulates
only those scenarios where message losses occur in the final round and initial states lead to a
definite decision. 
After the model becomes fully correct in this stage, GGMS relaxes the restriction to allow message losses in the last two rounds. GGMS then repeats simulation and learning until the model achieves full correctness under this relaxed setting. GGMS repeats this until it allows message losses in all rounds. Finally, GGMS relaxes to allow all possible initial states.

Such a design is based on the observation that there is less ambiguity when the protocol is in later rounds
and when the protocol starts with initial states that lead to definite decisions. By first simulating
in these scenarios, GGMS can better avoid the ``noise'' from ambiguity, and thus better propagate the
effects of freezing. Then, after such propagating has settled (i.e., the corresponding transitions reach
a high probability), GGMS can further propagate its effects by relaxing crash and initial states scenarios.
We present the details in \S\ref{sec:pseudo-overview}.

\subsection{Validating the Learned Model}
\label{sec:verify}
Exhaustive verification is what makes GGMS outputs \emph{provably correct}. Our verifier enumerates all $x^n$ initial state configurations ($x$ is the number of initial states and $n$ is the number of processes) and all message loss patterns consistent with at most $f$ crash failures, checking that the learned state machine satisfies all safety properties. A protocol is returned only if it passes this complete check. Verification is not the bottleneck in practice, as MCTS simulation dominates running time; details appear in \S\ref{sec:pseudo-validation}.
In the future, assuming MCTS will be optimized, we will
switch to more efficient verifiers like Z3

GGMS integrates verification into a counterexample-guided loop: after each episode, the verifier either confirms correctness (triggering termination or relaxation per \S\ref{sec:guided}) or returns counterexamples whose sampling rate is increased in subsequent episodes (\S\ref{sec:pseudo-retrain}).

\section{Evaluation}
\label{sec:eval}

\begin{table*}[t]
\centering
\resizebox{0.8\textwidth}{!}{
\begin{tabular}{|c|c|c|c|c|c|c|c|c|c|c|c|c|}
\hline
            & \multicolumn{3}{|c|}{Success rate} & \multicolumn{9}{|c|}{Running time (minutes) } \\

\hline
            & MCTS & MCTS+DFS & GGMS & \multicolumn{3}{|c|}{MCTS} & \multicolumn{3}{|c|}{MCTS+DFS} & \multicolumn{3}{|c|}{GGMS} \\
\hline
            & \multicolumn{3}{|c|}{} & avg & min & max & avg & min & max & avg & min & max \\
\hline
     ac-2-1 & 50\%  & 60\% & 100\% & 133 & 15 &301 & 52 & 25 & 96 & 15 & 8 &29 \\
\hline
     ac-3-2 & 0  & 40\% & 60\% & -- & -- & -- & 1413 & 722 & 1825  & 813 & 175 & 1441 \\
\hline
     ac-4-1 & 10\%  & 20\% & 100\% & 184 & 184 & 184 & 192 & 191 & 193 & 413 & 310 & 630 \\
\hline
     ac-4-2 & 0  & 0 & 70\% & -- & -- & -- & -- & -- & -- & 754 & 563 & 938 \\
\hline
     con-2-1 & 90\%  & 90\% & 100\% & 46 & 3 & 180 &  25 & 3 & 55 & 8 & 7 & 11 \\
\hline
     con-3-2 & 80\%  & 80\% & 100\% & 1106 & 49 & 2387 & 273 & 148 & 693 & 118 & 104 & 151 \\
\hline
     con-4-1 & 80\%  & 70\% & 100\% & 328 & 76 & 1576 & 348 & 146 & 654 & 268 & 253 & 280 \\
\hline
     con-4-2 & 30\%  & 40\% & 90\% & 678 & 470 & 922 & 1783 & 1277 & 2560 & 717 & 610 & 970 \\
\hline
     con-4-3 & 0  & 0 & 100\% & -- & -- & -- & -- & -- & -- & 3358 & 1857 & 5693 \\
\hline
\end{tabular}
}
\caption{Success rate and running time of different methods. We run each setting 10 times.}
\label{tab:summary}
 \vspace{-.2in}
\end{table*}

We apply GGMS to learn synchronous atomic commit and consensus protocols.
Atomic commit may look similar to consensus to some extent: Processes start with 
proposals ``abort'' or ``commit'' and try
to reach the same decision. However, atomic commit is different from consensus in
two ways: First, if a process proposes ``abort'', then the final decision of atomic commit
must be ``abort'' (consensus allows ``commit'' if another process proposes ``commit'').
Second, if a process crashes, the final decision could be ``abort'',
even if all processes propose ``commit'' (consensus requires ``commit''
in this case). Such differences lead to protocols that are different from consensus.

We document the formal
properties of these protocols in \S\ref{sec:formal-properties} and our experiment settings in Sections~\ref{sec:model-parameters} and \ref{sec:experiment-setting}.
Synchronous consensus is a well-studied field, with protocols such as FloodSet and Primary Backup~\citep{Bressoud1996Hypervisor}. Synchronous atomic commit is not well-studied
as far as we know, because well-known atomic commit protocols, such as two-phase commit (2PC)~\citep{gray1978notes},
target asynchronous environment.

For synchronous consensus, human knowledge tells that a setting $S$ with two internal states and $r>f$
($r$ is the number of rounds and $f$ is the number of failures allowed)
is feasible. We verified the same conclusion for synchronous atomic commit.
We report results on such feasible settings. As a sanity check,
we tested some infeasible settings and did not get a correct protocol.
In the following report, we use a triple prot-$n$-$f$ to represent the setting: prot represents the
protocol (``ac'' for atomic commit and ``con'' for consensus). $n$ is the total number of processes. We set $r=f+1$.
We compare GGMS with MCTS and MCTS+DFS to understand the effectiveness of DFS and Guided MCTS.
MCTS includes the techniques described in \S\ref{subsec:mcts}; MCTS+DFS
include \S\ref{subsec:mcts} and \S\ref{sec:dfs}; GGMS includes \S\ref{subsec:mcts}, \S\ref{sec:dfs}, and \S\ref{sec:guided}.

\begin{algorithm}[t]
\caption{Atomic commit found by GGMS (f=1)}
\label{alg:atomic-commit}
$\mathbf{Initialization:}$ \\
\hspace*{1em} $V_i \leftarrow \{v_i\} \quad \text{either init:commit or init:abort}$\\
$\mathbf{Round \ } 1:$ \\
\hspace*{1em} \text{if no message loss and every received $V_j$ is init:commit} \\
\hspace*{2em}    $V_i \leftarrow internal:a$\\
\hspace*{1em} \text{else}\\
\hspace*{2em}    $V_i \leftarrow internal:b$ \\
$\mathbf{Round \ } 2:$ \\
\hspace*{1em} \text{if every received $V_j$ is internal:a} \\
\hspace*{2em}      $V_i \leftarrow decision:commit$ \\
\hspace*{1em} \text{else}\\
\hspace*{2em}    $V_i \leftarrow decision:abort$ 

\end{algorithm}

 \vspace{-.05in}
\paragraph{Found protocols.}
For consensus, GGMS finds protocols that are similar to the FloodSet protocol. For atomic commit, GGMS finds protocols that modify FloodSet to adapt to the additional requirement of atomic commit. Concretely, these protocols treat ``lost message'' as ``abort'' in the first round
and ignore ``lost message'' in later rounds (see Algorithm~\ref{alg:atomic-commit} for an example). 
Though simple, this is
a new protocol as far as we know.

 \vspace{-.05in}
\paragraph{Success rate.}
Table~\ref{tab:summary} shows the success rate of different methods.
As DFS can guarantee eventual success, this success rate is based on a limited time, which is documented
in Sections~\ref{sec:experiment-setting}. 
For one setting, we run a method 10 times and report the number of times it can
succeed in finding a correct protocol. For MCTS and MCTS+DFS, if they had a low success rate at a certain setting,
we did not further try them on larger settings. As shown in this table,
in all settings, GGMS consistently achieves higher success rates than MCTS and MCTS+DFS. 
From the training logs, we find that MCTS is primarily bothered by the superposition problem,
and increasing training time does not help much. MCTS+DFS tries to address the superposition
problem by freezing transitions, but when it freezes wrong ones, it will need unfreezing, which
takes a lot of time.

 \vspace{-.05in}
\paragraph{Running time.}
Table~\ref{tab:summary} shows the running time to converge to a correct model (excluding unsuccessful runs).
GGMS' speed is faster than or similar to that of MCTS and MCTS+DFS in most of the settings.
Note that excluding unsuccessful runs
is favorable for MCTS and MCTS+DFS: In a difficult setting, MCTS and MCTS+DFS may fail
many trials, which do not count, but GGMS may succeed after a long run, which increases its
average running time.

Our further analysis
shows that in GGMS, the simulation time for MCTS dominates the overall time compared to the training and validation time.
For atomic commit, we also observe a significant variation in running time, usually due to incorrect freezing
leading to unfreezing.

 \vspace{-.05in}
\paragraph{Scalability.} 
The running time increases rapidly with larger settings, as expected. 
We currently implement GGMS with a single threaded Python program (except the verifier, which is parallelized). Techniques like C implementation
and parallelization should be able to bring a significant performance improvement.
However, the exponential growth pattern of the running time will probably persist. As a result, our speculation is that we may be able to scale to 8-10 processes for some protocols, but probably not to 100 processes. 

However, scaling to a large number of processes may not be necessary to design a distributed protocol. For human experts, a general practice is to 1) design a protocol for a small number of processes, then 2) distill the insights from these small instances and generalize these insights for a generic protocol with an arbitrary number of  processes, and finally 3) prove the correctness of the generic protocol. 
Therefore, as long as our tool can accomplish step 1), it may provide useful insights for human experts. Automating step 2) is our future work, probably involving the help
from LLMs. 
3) is well established in the distributed system community \citep{ma2019i4,yao2021distai,zhang2025basilisk}.


\begin{figure}[t]
    \centering
    \includegraphics[width=0.48\textwidth]{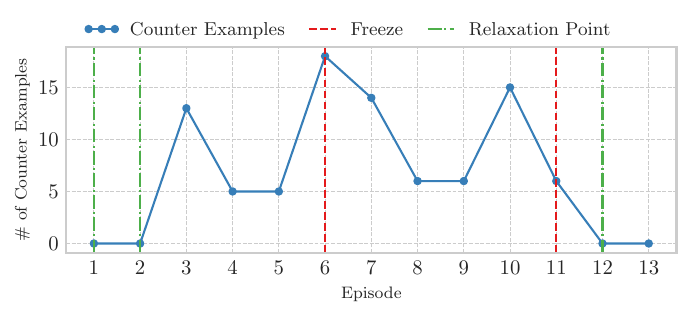}
    \caption{Number of counter examples for ac-3-2 (no unfreezing).}
    \vspace{-.1in}
    \label{fig:cs_plot}
\end{figure}

\begin{figure}[t]
    \centering
    \includegraphics[width=0.48\textwidth]{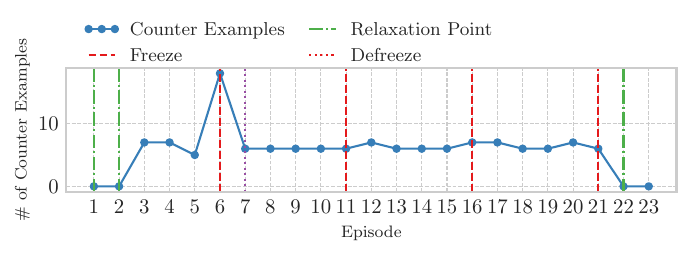}
    \caption{Number of counter examples for ac-3-2 (with unfreezing).}
     \vspace{-.1in}
    \label{fig:cs_plot_defreeze}
\end{figure}

\vspace{-.05in}
\paragraph{Number of counter examples during training.}
Figure~\ref{fig:cs_plot} shows the number of counterexamples found by our verifier in one setting.
At the beginning, the number of counterexamples is low, since guided MCTS focuses on some specific
scenarios. Then, after relaxation, guided MCTS starts to explore more scenarios, leading to more
counterexamples. GGMS performs two rounds of freezing, and the number of counterexamples drops to
zero. Finally, after more relaxation, GGMS does not find more counterexamples. 
Figure~\ref{fig:cs_plot_defreeze} presents another example.  GGMS first freezes a transition. It then discovers that, under this frozen condition, no correct model can be learned. Consequently, GGMS unfreezes the transition and continues training. The system eventually learns a correct model, though it requires more time compared to the example shown in Figure~\ref{fig:cs_plot}.

\vspace{-.05in}
\paragraph{Effects of ML models.}
Our current implementation uses the Transformer model as discussed in Section~\ref{sec:model-parameters}.
We also tried the MLP model. For the following settings ac-2-1, con-2-1, ac-3-2, and con-3-2, the success
rate of using the MLP model is 100\%, 70\%, 10\%, and 20\%, respectively, compared to 100\%, 100\%, 60\%, and 100\%
of using the Transformer model.

\section{Future work}
\label{sec:future}

\vspace{-.05in}
\paragraph{Integrating with LLMs.} 
Recent work on LLM-based reasoning and code generation, including AlphaProof and AlphaEvolve~\citep{AlphaProof,AlphaEvolve}, 
demonstrates the power of combining learned models with formal verification or evaluation oracles. 
Our experiments with GPT and Gemini (see Appendix~\ref{sec:GPT} for details) have confirmed this: they could find the FloodSet protocol for consensus, but
often needed help (we manually verify the protocol and provide counter examples) for atomic commit, which has no known solutions. Therefore,
we believe that it is still challenging for LLMs to fully accomplish the discovery, especially for more complicated protocols, but their
capabilities may be greatly enhanced by external tools like this work. In the future, we will explore whether we can fuse the strengths of LLMs and our work.
In particular, our work can provide verification, (counter-)examples, and exploration for LLMs, and
LLMs can synthesize a generic protocol that works for an arbitrary number of processes, which is 
a challenge for our work.


\vspace{-.1in}
\paragraph{Extending this work.}
\S\ref{sec:modeling} discusses several limitations of this work.
Allowing a process to make a decision early before the last round is a straightforward extension of this work.
Allowing a process to send different messages to different processes will significantly increase the search space.
As a middle ground, we may allow a process to send its state to a subset of processes.
To model asynchronous networking, we need to change our simulator to include more complicated message loss scenarios.
To model Byzantine processes, we can incorporate adversarial models so that two models play against each other.


\section{Related Work}
\label{sec:relate}

Modern ML systems that combine self-play with search have achieved superhuman performance in games with both full and partial observability~\citep{mnih2015human,silver2017mastering,li2020suphx,bard2020hanabi,brown2019superhuman}. Designing distributed protocols is harder in two ways: processes have only local views (partial observability), and we require guaranteed correctness under faults and message loss rather than high average-case performance. To the best of our knowledge, prior work does not learn distributed protocols with formal correctness guarantees.

Closer to our problem, \citet{khanchandani2021learning} use self-play to learn algorithms for a distributed directory problem, while \citet{Zhang_Li_Chen_Li_Zhang_Zhao_Liu_Chen_2025} and \citet{Belardinelli2024VerificationMAS} combine reinforcement learning with model checking for multi-agent systems. These approaches learn policies or algorithms and use verification as a shaping signal or evaluation tool. In contrast, GGMS uses model checking as a hard oracle that accepts or rejects entire candidate protocols.

\vspace{-.1in}
\paragraph{Positioning among synthesis and verification approaches.} GGMS occupies a distinct niche among related methods. Classical program synthesis via counterexample-guided inductive synthesis (CEGIS)~\citep{SolarLezama2008Sketching} iteratively refines candidates using counterexamples, but operates on sequential programs and relies on SMT solvers to propose candidates—an approach that struggles with the combinatorial explosion of distributed state spaces. 
Automata learning methods~\citep{Higuera2005Study,Heule2010ExactDFA} infer state machines from labeled traces, but assume access to a correct protocol's executions—precisely what we aim to discover. Verification frameworks for distributed systems~\citep{hawblitzel2015ironfleet,ma2019i4,yao2021distai,zhang2025basilisk} prove correctness of human-designed protocols but do not synthesize them.

GGMS bridges these gaps: it uses MCTS-guided search to navigate the protocol space efficiently, model checking as a hard correctness oracle (not just a shaping reward), and counterexample-driven retraining in a CEGIS-like loop—but applied to multi-agent state machines under adversarial fault models. GGMS does not depend on existing solutions, which
makes it complementary to LLM based approaches. Appendix~\ref{app:competing-approaches} provides further detailed comparisons across these approaches.

\section{Conclusion}

This work explores the new area of learning provably correct
distributed protocols with partial observability. 
We propose GGMS, which combines model checking to ensure correctness,
DFS to ensure eventual convergence, and guided MCTS to accelerate convergence.
As a preliminary exploration, we further discuss possible future directions.

\bibliography{reference}
\bibliographystyle{icml2026}

\newpage
\appendix
\onecolumn

\clearpage
\begin{appendix}

\setcounter{figure}{0}
\setcounter{table}{0}
\renewcommand{\thefigure}{S\arabic{figure}}
\renewcommand{\thetable}{S\arabic{table}}

\section{Formal definitions of investigated protocols}
\label{sec:formal-properties}

The complete definitions of each state are provided in Table~\ref{table:state-definitions-ac} and Table~\ref{table:state-definitions-consensus}, for the atomic commit and consensus protocols, respectively. Based on these state definitions, we formally define the properties of each protocol following prior work. Let $N$ denote the set of all processes. Each process has an initial state $b_n$, and a set of final decisions, $F_n$. Let $M$ represent all messages exchanged between processes, and let $L=\exists m \in M (m=Lost)$ indicate that there exists at least one lost message.

\begin{table}[htbp]
\footnotesize
\centering
\begin{tabular}{ll}
\hline
\textbf{State} & \textbf{Meaning} \\
\hline
\texttt{init:abort}         & Initial state representing intent to abort \\
\texttt{init:commit}         & Initial state representing intent to commit \\
\texttt{internal:a}     & Internal state \\
\texttt{internal:b}     & Internal state \\
\texttt{decision:abort}     & Final decision to abort \\
\texttt{decision:commit}     & Final decision to commit \\
\hline
\end{tabular}
\caption{State Definitions for Atomic Commit Protocol}
\label{table:state-definitions-ac}
\end{table}

\begin{table}[htbp]
\footnotesize
\centering
\begin{tabular}{ll}
\hline
\textbf{State} & \textbf{Meaning} \\
\hline
\texttt{init:0}         & Initial state representing intent to commit 0 \\
\texttt{init:1}         & Initial state representing intent to commit 1\\
\texttt{internal:a}     & Internal state \\
\texttt{internal:b}     & Internal state \\
\texttt{decision:0}     & Final decision to commit 0 \\
\texttt{decision:1}     & Final decision to commit 1\\
\hline
\end{tabular}
\caption{State Definitions for Consensus Protocol}
\label{table:state-definitions-consensus}
\end{table}

\paragraph{Atomic commit protocol properties.} 
Rule P1 states that no two processes should reach conflicting decisions. Rule P2 ensures that if all processes initially intend to commit and there are no message losses, then all processes must reach a commit decision. Rule P3 indicates that if any process initially intends to abort, then all processes must reach an abort decision. Rule P4 requires that every process must eventually reach a final decision.
\begin{enumerate}
\item[\textbf{P1.}]
\begin{multline}
\neg \exists n, m \in N \; \exists f_n, f_m \in F \;\\ \big( 
f_n = \texttt{decision:abort}  \land f_m = \texttt{decision:commit} \big)
\end{multline}

\item[\textbf{P2.}]
\begin{multline}
\bigl(\forall n\in N,\; b_n=\text{\ttfamily init:commit}\bigr)\land(\neg L)\\
\Rightarrow\bigl(\forall n\in N,\ \forall f_n\in F_n,\ f_n=\text{\ttfamily decision:commit}\bigr).
\end{multline}

\item[\textbf{P3.}]
\begin{multline}
\left( \exists n \in N,\, b_n = \texttt{init:abort} \right)\\ \Rightarrow  {} 
\left( \forall n \in N,\, \forall f_n \in F_n,\, f_n = \texttt{decision:abort} \right)
\end{multline}

\item[\textbf{P4.}]
\begin{multline}
\forall n \in N, \exists f_n \in F \; \big( f_n = \texttt{decision:abort}
 \lor f_n = \texttt{decision:commit} \big)
\end{multline}

\end{enumerate}

Note that the well-known atomic protocols like two-phase commit (2PC) and three-phase commit (3PC) make
the asynchronous networking assumptions that every message has a non-zero chance of being lost. As discussed
in Section~\ref{sec:modeling}, our current implementation supports only synchronous networking assumptions.
Such a difference leads to a difference in the protocol properties: In the asynchronous versions, we can prove
that there always exist scenarios that a correct process cannot decide. In the synchronous versions,
such scenarios do not exist.

\paragraph{Consensus protocol properties.} 
Rule P1 states that no two processes should reach conflicting decisions. Rules P2 and P3 assert that if all processes start with the same initial proposal, then all final decisions must match that proposed value. Rule P4 requires that every process must reach a decision by the end of the protocol.
\begin{enumerate}

\item[\textbf{P1.}]
\begin{multline}
\neg \exists n, m \in N \; \exists f_n, f_m \in F \; \big( 
f_n = \texttt{decision:0}  \land f_m = \texttt{decision:1} \big)
\end{multline}

\item[\textbf{P2.}]
\begin{multline}
\left( \forall n \in N, b_n = \texttt{init:1} \right) \Rightarrow  {}
\left( \forall f_n \in F, f_n = \texttt{decision:1} \right)
\end{multline}

\item[\textbf{P3.}]
\begin{multline}
\left( \forall n \in N, b_n = \texttt{init:0} \right) \Rightarrow  {}
\left( \forall f_n \in F, f_n = \texttt{decision:0} \right)
\end{multline}

\item[\textbf{P4.}]
\begin{multline}
\forall n \in N, \exists f_n \in F \; \big( f_n = \texttt{decision:0}
 \lor f_n = \texttt{decision:1} \big)
\end{multline}

\end{enumerate}

\section{Proof of Theorem}
\label{sec:proof}

\setcounter{theorem}{1}
\begin{theorem}
    In a feasible setting, assuming 1) the scenario $Sc$ of this particular simulation leads to a definite decision (i.e, no ambiguity),
    and 2) for each pair of (round, procID), this step only fixes one transition
    (round, procID, inputA) $\rightarrow$ B, then there exists a correct state machine with all
    the fixed transitions ($Transitions_{fix}$) in this step.
\end{theorem}

\begin{proof}
    Assuming there exists a correct state machine $SM_0$, we prove that we can construct
    a state machine $SM_{r-1}$, such that, 1) $SM_0$ and $SM_{r-1}$ are equivalent, i.e.,
    for every scenario $S$, every non-crashed process makes the same decision with $SM_0$ and $SM_{r-1}$;
    and 2) $SM_{r-1}$ includes all the transitions in $Transitions_{fix}$.

    We construct by induction of $r-1$ steps ($r$ is the number of maximal rounds).
    Assuming $SM_{i-1}$ is already constructed ($1<i<r$), we construct 
    $SM_i$ in the following way. We search for procID, such that transition $[round=i, procID, inputA] \rightarrow B$
    exists in $Transitions_{fix}$ and $[round=i, procID, inputA] \rightarrow C$ exists in $SM_{i-1}$ and B is different
    from C. Then we swap B and C in $SM_{i-1}$ to get $SM_i$: 1) For any input, if there exists a transition 
    $[i, procID, input] \rightarrow C/B$ in $SM_{i-1}$, we change it to $[i, procID, input] \rightarrow B/C$ in $SM_i$; 2) In round i+1, for any procID2, if there exists a transition $[i+1, procID2, input] \rightarrow D$ in $SM_i$, where the input vector includes C (or B) from procID, we change B into C and C into B in the input vector.
    
    First, we can prove that $SM_i$ includes transitions in $Transitions_{fix}$ whose round is smaller than
    i+1. We prove by induction. When $i=1$, $SM_1$ is constructed from $SM_0$. Under $Sc$,
    $SM_0$ and $Transitions_{fix}$ must reach the same $[round=1, procID, inputA]$ in the first round, 
    since the input is from initial states. And if a $[round=1, procID, inputA]$ transit to different states
    in $SM_1$ and $Transitions_{fix}$, our construction changes the output of
    the transition in $SM_1$ to match that in $Transitions_{fix}$. Then for the later round $i$,
    we can prove it in the same way. Under $Sc$, $Transitions_{fix}$ and $SM_{i-1}$ must reach
    the same $[round=i, procID, inputA]$ in round $i-1$, as they have the same transitions for
    $Sc$. Then if $[round=i, procID, inputA]$ transit to different states in $Transitions_{fix}$ and $SM_{i-1}$,
    our construction forces them to be the same in $Transitions_{fix}$ and $SM_i$
    Note that this only works if for each pair of (round, procID), $Transitions_{fix}$ only
    includes one transition. Otherwise, the construction may need to swap multiple pairs
    of values, and they may conflict,
    which means the construction may not be possible  (e.g., we cannot both swap B and C and swap B and D).
    
    Second, we can prove that $SM_{i-1}$ and $SM_i$ are equivalent. Assuming a simulation applies transition
    $(i, procID, input) \rightarrow B/C$ in $SM_{i-1}$, the simulation will get the same input for procID in round i when applying $SM_i$,
    as $SM_{i-1}$ and $SM_i$ have the same set of transitions before round i. Then the simulation will apply 
    $(i, procID, input) \rightarrow C/B$ in $SM_i$, i.e., output of $SM_{i-1}$ and $SM_i$ swap B and C for procID.
    However, since our construction also swaps B and C in the input of round i+1 from procID, we can know that in round i+1, every process will still transit to the same state. Therefore, $SM_{i-1}$ and $SM_i$ are equivalent.

    With the above steps,
    we can construct a state machine $SM_{r-1}$ that is equivalent to $SM_0$ and that includes all the transitions
    from $Transitions_{fix}$ till round r-1. In round r, every
    state machine transits to the decision state. We can prove that $SM_{r-1}$ must have the same transitions as those in
    $Transitions_{fix}$ for round r, with no need for swapping. This is because, for $Sc$ that generates
     $Transitions_{fix}$, $SM_{r-1}$, and $Transitions_{fix}$ will apply the same transitions till round r-1, so every process
     should have the same input at the beginning of round r. If $SM_{r-1}$ and $Transitions_{fix}$ have different transitions for the same input
     in round r, they will lead to different decisions of some processes.
     This contradicts our assumption that this scenario leads to a definite decision.

    Therefore, we have proved that $SM_{r-1}$, which is equivalent to $SM_0$ and must be correct, has all the transitions of
   $Transitions_{fix}$.
     
\end{proof}

\section{Implementation Details}
\label{sec:pseudo}

\subsection{Overview}
\label{sec:pseudo-overview}

\begin{algorithm}[t]
\caption{Pseudo code of GGMS}
\label{alg:ggms}
\SetCommentSty{textit}
\SetKwProg{Fn}{Function}{:}{end}
\SetKwComment{MyComment}{/* }{ */}
$model$ $\gets$ init\_model()\; \label{line:init}
$phase\_ID$ $\gets$ 0\;
$training\_buffer$ $\gets$ []\;
$failed\_scenarios$ $\gets$ []\;
$freeze\_list$ $\gets$ []\;
\MyComment{Each main loop iteration is an episode}
\While{true}{
    \For{$i \gets 1$ \KwTo $100$} {
        $scenario$ $\gets$ \texttt{sample\_scenario}($phase\_ID$, $failed\_scenarios$)\; \\
        $training\_data, reward$ $\gets$ \texttt{run\_mcts}($scenario$, $model$)\; \\
        $training\_buffer.\texttt{append}(training\_data)$\; \\
        \texttt{determine\_unfreeze}($reward$, $freeze\_list$)\; \\
    }
    \texttt{determine\_freezing}($training\_buffer$, $freeze\_list$)\; \\
    \texttt{update\_model}($training\_buffer$, $model$)\; \\
    $failed\_scenarios$ $\gets$ \texttt{validate}($phase\_ID$, $model$)\; \\
    \If{$failed\_scenarios == [\,]$}{
        \If{$phase\_ID == lastPhase$}{
            \texttt{terminate}\;
        }
        \Else{
        $phase\_ID \gets phase\_ID + 1$\;
        }
    }

}
\end{algorithm}

Algorithm~\ref{alg:ggms} presents the high-level pseudo code of GGMS.

\texttt{model} is a neural network representing the state machine we want to learn. As discussed in
Section~\ref{sec:modeling}, it takes $[round,procID,inputStates]$ as the input and
outputs a $newState$. In fact, to facilitate learning, we let it output a probability
for each potential value of $newState$. During inference, GGMS will choose the value
with the highest probability.

\texttt{phase\_ID} is used to implement the guided MCTS (Section~\ref{sec:guided}).
When set to 0, it means that GGMS will only simulate scenarios with definite initial
states and message losses in the last round. When set to 1, it means that GGMS will
simulate scenarios with definite initial states and message losses in the last two rounds,
etc. Finally, when set to $r$ (the total number of rounds), it means that GGMS
can use any initial state and message losses in any round.

\texttt{training\_buffer} is a bounded buffer to store training data collected during
MCTS. Each item in the buffer records the probability of a transition from $[round,procID,inputStates]$
to one potential output value.

\texttt{failed\_scenarios} records scenarios that caused prior validation to fail. As discussed,
GGMS will use such scenarios to retrain the model.

\texttt{freeze\_list} is used to implement DFS (Section~\ref{sec:dfs}). Like a conventional DFS implementation, \texttt{freeze\_list} is a stack, and each item in the stack represents one frozen transition 
$[round,procID,inputStates] \rightarrow newState$. Each item also records whether other values have
been frozen for the same $[round,procID,inputStates]$ in the past. 

The whole algorithm works in multiple episodes. In each episode, it runs MCTS on 100 scenarios (line 7).
Each scenario is either randomly chosen from scenarios allowed by the current phase or from past failed
scenarios (line 8). Running MCTS on the scenario will generate some training data, which will be
added to the training buffer, and a reward (lines 9-10). GGMS will determine whether it needs to unfreeze depending
on the reward (line 11).

Then, after simulating 100 scenarios, GGMS will determine whether it needs to freeze more transitions based on
the new training data (line 13). Note that unfreeze and freeze are determined at different timings: If MCTS
cannot find a good model for one scenario, that is already enough to trigger unfreeze, and that is why unfreeze
is determined after simulating every scenario. However, determining freezing often requires information
from multiple scenarios, which can lead to potentially different transitions for the same input, and that is
why GGMS determines freezing after trying a number of scenarios.

Then GGMS updates the model using the new training data (line 13) and then validates the new model (line 14).
If validation does not find any failed scenarios, GGMS will either terminate if this is the last phase, or
proceed to the next phase otherwise (lines 16-23).

\subsection{Use counterexamples to retrain}
\label{sec:pseudo-retrain}

\begin{algorithm}[t]
\caption{Pseudo code of choosing a scenario}
\label{alg:sample-scenario}
\SetCommentSty{textit}
\SetKwProg{Fn}{Function}{:}{end}
\SetKwComment{MyComment}{/* }{ */}
\Fn{\texttt{sample\_scenario}($phase\_ID$, $failed\_scenarios$)}{
    $u$ $\gets$ \texttt{Uniform}(0, 1)\; \\
    
    \If{$u < 0.3$ \textnormal{ or } $failed\_scenarios == [ ]$}{
        $scenario$ $\gets$ \texttt{uniform\_sample}(\texttt{generate\_all\_scenarios}($phase\_ID$))\; \\
    }
    \Else{
        $scenario$ $\gets$ \texttt{uniform\_sample}($failed\_scenarios$)\; \\
    }
    
    \KwRet $scenario$\; \\
}

\end{algorithm}

Algorithm~\ref{alg:sample-scenario} presents the details of how GGMS selects scenarios
to simulate. It has a 70\% chance to choose a failed scenario in the past, if any, and
30\% chance to randomly choose a scenario allowed by the current phase.

\subsection{Enhancing MCTS with DFS}
\label{sec:pseudo-dfs}

\begin{algorithm}[t]
\caption{Pseudo code of determining unfreezing}
\label{alg:determine-unfreeze}
\SetCommentSty{textit}
\SetKw{KwBreak}{Break}
\SetKw{KwContinue}{Continue}
\SetKwProg{Fn}{Function}{:}{end}
\SetKwComment{MyComment}{/* }{ */}
\Fn{\texttt{determine\_unfreeze}($reward$, $freeze\_list$)}{
    \label{func:determine_unfreeze_impl}
    \MyComment{Unfreeze when reward is negative and some frozen transition was activated}
    \If{$reward < 0$ \textnormal{and} \texttt{has\_activated}($freeze\_list$)}{
        \MyComment{At least one entry can be popped out}
        \While{true}{
            $entry$ $\gets$ $freeze\_list.\texttt{pop\_one\_activated}()$\; \\
            \If{$entry$ is not fully explored}{
                \texttt{freeze\_to\_new\_value}($entry$)\; \\
                $freeze\_list.\texttt{push}(entry)$\; \\
                \KwBreak\; \\
            }
        }
    }
}
\end{algorithm}

Algorithm~\ref{alg:determine-unfreeze} presents the logic for determining whether to unfreeze a frozen transition and, if so, which transition to unfreeze.
Our current implementation uses the condition that the reward is negative, which means
that MCTS cannot find a model to reach correct decisions for this scenario, and at least one frozen transition
is activated (line 2). In our current implementation, MCTS is given enough time to fully explore
all state machines relevant to this scenario, so the negative reward is an accurate condition
to trigger unfreezing. However, such exhaustive search scales poorly.
For better scalability, we may replace this condition
with a Z3-style validation to prove that, given the frozen transitions, the model can
never reach correct decisions for the corresponding scenario, no matter what other
transitions this model includes. We will explore this in the future.

If the condition is met, GGMS unfreezes transitions in the DFS manner. It pops an activated transition
from the stack (line 4). If the corresponding input of the transition is not fully explored (line 5),
which means that GGMS has not tried to freeze it to all the possible values, GGMS will try to freeze
it to a value that has not been explored (line 6) and push this entry back into the stack.
Otherwise, GGMS will keep popping until it can find such an entry.

\begin{algorithm}[t]
\caption{Pseudo code of determining freezing}
\label{alg:determine-freeze}
\SetCommentSty{textit}
\SetKwProg{Fn}{Function}{:}{end}
\SetKwComment{MyComment}{/* }{ */}
\Fn{\texttt{determine\_freezing}($training\_buffer$, $freeze\_list$)}{
    \label{func:determine_freezing_impl}
    \MyComment{Find inputs whose outputs have close probabilities and freeze one}
    $tmp$ $\gets$ \texttt{find\_ambiguous\_inputs}($training\_buffer$, $p\_min{=}0.2$, $diff\_max{=}0.1$)\; \\ \label{line:df_find}
    \MyComment{Sort: later round first, then fewer lost messages first}
    $tmp$ $\gets$ \texttt{sort}($tmp$, \texttt{key=}[\texttt{round\_desc}, \texttt{lost\_msgs\_asc}])\; \\

    \If{$tmp \neq []$}{
        $cand$ $\gets$ $tmp[0]$\; \\
        $entry$ $\gets$ $cand.freeze\_outputs$\; \\
        $freeze\_list.\texttt{push}(entry)$\; \\
    }
}
\end{algorithm}

Algorithm~\ref{alg:determine-freeze} presents the logic of determining whether to freeze
a new transition and, if so, which one to freeze. Our current implementation uses
the heuristics that if multiple outputs for the same input have a probability larger
than 0.2 and the difference between their probabilities is within 0.1, then GGMS considers
them as targets for freezing (line 2). If multiple such inputs exist, GGMS sorts
them based on their round number and the number of lost messages in the input and chooses
the one with the highest round number and the lowest number of lost messages as input (lines 3-5).
This is based on our experience in the importance of such transitions. Note that this
heuristic does not affect the eventual convergence of DFS. Finally, GGMS pushes the newly
determined frozen transition into \texttt{freeze\_list}.

\subsection{Brute-force validation}
\label{sec:pseudo-validation}

\begin{algorithm}[t]
\caption{Pseudo code of validation}
\label{alg:validate}
\SetCommentSty{textit}
\SetKwProg{Fn}{Function}{:}{end}
\SetKwComment{MyComment}{/* }{ */}
\Fn{\texttt{validate}($phase\_ID$, $model$)}{
    \label{func:validate_impl}
    \MyComment{Enumerate all patterns for the phase and simulate (can run in parallel)}
    $init\_state\_patterns $ $\gets$ \texttt{gen\_all\_inputs}($phase\_ID$)\; \\ 
    $msg\_loss\_patterns$ $\gets$ \texttt{gen\_loss\_patterns}($phase\_ID$)\; \\ 
    $all\_scenarios$ $\gets$ $init\_state\_patterns \times msg\_loss\_patterns$\; \\ 
    $failed$ $\gets$ []\; \\ 

    \ForEach{$scenario \in all\_scenarios$ \textnormal{ in parallel}}{
        $ok$ $\gets$ \texttt{simulate}($model$, $scenario$)\; \\ 
        \If{\textnormal{not} $ok$}{
            $failed.\texttt{append}(scenario)$\; \\
        }
    }
    \KwRet $failed$\; \label{line:v_return2}
}
\end{algorithm}

Algorithm~\ref{alg:validate} presents our brute-force algorithm to validate whether the model
is fully accurate. It enumerates all the possible scenarios by doing a cross product of
all the possible initial state patterns and all the message loss patterns (lines 2-4). Generating all
the possible initial state patterns is straightforward: Suppose that there are $N$ processes
and $x$ possible initial state values. GGMS enumerates all ways to assign initial states to
different processes, generating a total of $x^N$ patterns. Generating all message loss
patterns is more complex. GGMS first generates all possible process crash patterns given
$phase\_ID$. Then, GGMS generates all message loss patterns based on such crash patterns.
In any round before a process crashes, GGMS marks all its messages as not lost. In any round
after a process crashes, GGMS marks all its messages as lost. In the same round as a process
crashes, its messages may or may not be lost, and thus GGMS enumerates all such possibilities.
Finally, GGMS performs a cross product of the possible message loss patterns of each process.

Then GGMS simulates each of the scenarios to see whether any of them will cause the model
to reach incorrect decisions (lines 6-11). Our current implementation parallelizes such
simulation of multiple scenarios. As discussed in Section~\ref{sec:eval}, validation
is not the bottleneck of our current implementation. In the future, we plan to replace
it with a Z3-based validation for better scalability.

\subsection{Monte-Carlo Tree Search}
\label{sec:pseudo-mcts}

\begin{algorithm}[H]
\caption{Pseudo code of Monte-Carlo Tree Search}
\label{alg:mcts}
\SetCommentSty{textit}
\SetKwFunction{simulate}{simulate}
\SetKwProg{Fn}{Function}{:}{end}
\SetKwComment{MyComment}{/* }{ */}
\Fn{\texttt{run\_mcts}($scenario$, $model$, $freeze\_list$)}{
    \MyComment{$protocol$: protocol representation}
    $protocol$ $\gets$ \texttt{initial($scenario$)}\; \\ \label{line:ep_init}
    $buffer$ $\gets$ []\;\\

    \While{\textnormal{not} \textnormal{$protocol$.is\_done}()}{
        \MyComment{Run Monte-Carlo tree search from current state}
        $P$ $\gets$ \texttt{simulate}($protocol$, $model$, $freeze\_list$)\; \\
        
        \MyComment{Pick transition based on simulation results}
        $transition$ $\gets$ \texttt{select($P$)}\; \\
        
        \MyComment{Advance to next round}
        \texttt{$protocol$.step}($transition$, $scenario$)\; \\

        \MyComment{Non-zero reward only at the last round}
        $reward$ $\gets$ \texttt{get\_reward}()\;\\

        \MyComment{Record (state, simulated probabilities)}
        $buffer.\texttt{append}(\langle \texttt{current\_state}(), P \rangle)$\; \label{line:store}\\
    }
    \label{line:end}
    \KwRet $buffer$, $reward$\; \\
}
\Fn{\simulate{$protocol$, $model$, $freeze\_list$}}{
    \label{func:simulate}
    \label{line:sim_start}
    \For{$i \gets 0$ \KwTo $iter$}{ 
        \MyComment{Store visited path during simulation}
        $visited$ $\gets$ []; \\
        \While{$protocol$ \textnormal{is not done}}{ 
            \MyComment{Select transitions and message losses based on Equation~\ref{eq:ucb}}
            $transition$ $\gets$ \texttt{select\_transition($model$, $freeze\_list$)}\;\label{line:selectact} \\
            $loss$ $\gets$ \texttt{select\_message\_loss()}\; \label{line:selectcrash} \\
            $protocol$.\texttt{next($transition, loss$)}\; \\ \label{line:step}
            $visited$.\texttt{append($transition, loss$)}\;
        }
        $reward$ $\gets$ \texttt{get\_reward()}\; \label{line:evaluation}\\
        \texttt{backup($reward, visited$)}\; \label{line:backup}\\
    }
    \label{line:sim_end}
}
\end{algorithm}

Algorithm~\ref{alg:mcts} shows the details of the MCTS simulation. At the beginning (line~\ref{line:ep_init}), it initializes a new $protocol$ object, representing the full protocol execution state with a specific $scenario$. Then it simulates the protocol starting from the current state (line~\ref{func:simulate}). We will describe the \texttt{simulate} function in detail later. It returns a probability distribution over transitions for the corresponding state, which is then stored in the training buffer (line~\ref{line:store}). Next, it selects a transition based on the simulation results and moves to the next protocol state. This process repeats until the protocol terminates. 

The \texttt{simulate} function primarily consists of four components. Figure~\ref{fig:mcts} illustrates an example of MCTS: it shows the search tree constructed during the simulation of one episode, and how the path is selected through it. We will introduce these main components in both Algorithm~\ref{alg:mcts} and Figure~\ref{fig:mcts}.

\begin{itemize}[leftmargin=*]
    \item \textbf{Selection} (line~\ref{line:selectact} and line~\ref{line:selectcrash}). 
    The algorithm will traverse the tree from the root.
    When selecting a transition from the current node, if a transition is in the \texttt{freeze\_list},
    GGMS selects it directly. Otherwise, GGMS selects a transition
     based on a balance between exploitation and exploration based on Upper Confidence Bound score as shown in Equation~\ref{eq:ucb}. The transition with the highest $U(s,a)$ will be selected in the simulation. $Q(s,a)$ is the accumulative average rewards that taking transition $a$ in state $s$ during simulation. $P(s,a)$ is the probability of choosing transition $a$ in state $s$, given by the policy network. $N(s,a)$ represents the number of times transition $a$ has been chosen in state $s$ during tree search simulations. $c_{puct}$ is a constant parameter that controls the balance between exploitation and exploration. 
    The selection will terminate when the protocol terminates. All visited transitions during selection will be stored. 
    \begin{equation}
    U(s, a) = Q(s,a) + c_{puct} \cdot P(s, a) \cdot \frac{\sqrt{\sum_b N(s, b)}}{1 + N(s, a)}
    \label{eq:ucb}
    \end{equation}
    
    The selection of transitions and message losses follows the same logic but with opposing objectives: the transition selector aims to maximize the final reward, while the message loss selector aims to minimize it. Unlike transition selection, message loss selection relies solely on accumulated rewards without guidance from the network. As shown in Figure~\ref{fig:mcts}, within a single iteration of simulation, the search may follow a path such as C1~$\rightarrow$~A1~$\rightarrow$~C3~$\rightarrow$~A4 (illustrated with solid arrows). In other iterations, different paths may be selected.

    Note that, as shown in this algorithm, although each simulation has a targeted scenario, it will explore other reachable scenarios during its
    search to avoid obviously wrong transitions to other scenarios.
    
    \item \textbf{Expansion}.
    When an unexpanded node is reached, all of its available child nodes will be added to the tree for further simulation. For each new node, the visited count and the accumulative reward will be initialized to 0.
    
    \item \textbf{Evaluation} (line~\ref{line:evaluation}).
    When the $protocol$ reaches termination, the final reward is computed based on a predefined reward function (+1 if no correctness property is violated and -1 otherwise). As shown in Figure~\ref{fig:mcts}, MCTS selects a particular path in this simulation (C1~$\rightarrow$~A1~$\rightarrow$~C3~$\rightarrow$~A4). At round~1, processes P2 and P3 choose the transition \texttt{internal:a(2)}. At round~2, the only living process, P3, selects transition \texttt{decision:1(1)}. According to the reward function, this selection results in a final reward of $+1$.
    
    \item \textbf{Backup} (line~\ref{line:backup}).
    The final reward is backed up along the search path stored in \texttt{visited}. The visit count and accumulated rewards of the corresponding nodes are updated accordingly. For example, suppose the selected path in this iteration is C1~$\rightarrow$~A1~$\rightarrow$~C3~$\rightarrow$~A4, and the final reward is $+1$. In this case,  the visit count of all nodes along the path is incremented by 1. The accumulated rewards of nodes A1 and A4 are increased by 1, indicating the protocol has made the right transitions, while those of nodes C1 and C3 are increased by the opposite, $-1$, indicating the adversary is not able to defeat the protocol. These values will be used to compute the probability of each activated transition at the end of \texttt{simulate}.

\end{itemize}

\section{Examples of superposition problem in MCTS}
\label{sec:examples-convergence}

We use the consensus protocol as an example to illustrate how superposition, that is, combining
transitions from multiple correct versions of the protocol can occur and affect the
convergence of MCTS. Keep in mind that consensus requires that 1)
every process makes the same decision, and 2) if the initial input
to every process is \texttt{init:0}, then the decision must be \texttt{decision:0};
if the initial input
to every process is \texttt{init:1}, then the decision must be \texttt{decision:1};
if some processes has \texttt{init:0} as the input and some have \texttt{init:1}, then the decision
could be either \texttt{decision:0} or \texttt{decision:1}.

Human knowledge tells that, inside such a protocol, each process
should use an internal state to record its intention, and multiple processes should exchange
their intentions to resolve divergence among processes. For example, if a process
observes that all processes have \texttt{init:0} in the first round, it may change its
internal state to \texttt{internal:a}, indicating that it intends to go to decision 0. Note that
if a process
observes some \texttt{init:0} and some \texttt{Lost}, it should transit to \texttt{internal:a}
as well, since \texttt{Lost} may be from a process with \texttt{init:0}. Similarly, if a process observes \texttt{init:1}
or \texttt{Lost}, but no \texttt{init:0}, then it can transit to \texttt{internal:b}, indicating
that it intends to go to decision 1. However, if a process observes both \texttt{init:0} and \texttt{init:1},
it can transit to either \texttt{internal:a} or \texttt{internal:b}. Processes can
exchange such intentions for multiple rounds until a consensus can be reached.

There are at least two reasons for the existence of multiple versions of the correct protocols.
First, without human knowledge, the protocol may assign certain meanings to arbitrary internal states,
creating multiple equivalent protocols. For example, while the above example uses \texttt{internal:a}
for intention \texttt{decision:0} and \texttt{internal:b} for intention \texttt{decision:1},
we can swap this mapping to create an equivalent protocol. This creates a problem for MCTS. When
MCTS simulates the scenario with all processes having \texttt{init:0} as input; it may find
that it is feasible for a process to transit to \texttt{internal:a} in this case, and finally transit
to \texttt{decision:0}. When
MCTS simulates the scenario with all processes having \texttt{init:1} as input; it may also find
that it is feasible for a process to transit to \texttt{internal:a} in this case, and finally transit
to \texttt{decision:1}. Although the solution to each individual scenario is correct, combining them is
incorrect, since we should not let the all \texttt{init:0} scenario and the all \texttt{init:1} scenario
transit to the same internal state, as there is no way to distinguish them in the later rounds. 

The second reason comes from the inherent ambiguity allowed by the protocol, that is,
if some processes have \texttt{init:0} as input and some have \texttt{init:1}, then the decision
could be either \texttt{decision:0} or \texttt{decision:1}. To give a concrete example about how this
causes problems for MCTS,
suppose that there are three processes participating in this protocol. Their input states are
[\texttt{init:0}, \texttt{init:0}, \texttt{init:1}]. Suppose Process 0
crashes in the first round; its message is received by Process 1, but not Process 2 (Scenario 1).
So Process 1's input is [\texttt{init:0}, \texttt{init:0}, \texttt{init:1}] (Input A) and Process'2
input is [Lost, \texttt{init:0}, \texttt{init:1}] (Input B). Due to the ambiguity of the initial input,
we know that Input A and Input B can transit to either \texttt{internal:a} or \texttt{internal:b}, as long
as they transit to the same internal state. If MCTS only simulates this scenario, it may find this constraint
and thus give
$[A \rightarrow \texttt{internal:a}]$ and $[B \rightarrow \texttt{internal:a}]$ a
higher probability than $[A \rightarrow \texttt{internal:b}]$ and $[B \rightarrow \texttt{internal:b}]$. However, Input A and Input B may appear separately in other scenarios as well.
For example, if Process 0 does not crash in the above example, all processes will have Input A but no Input B (Scenario 2); if Process 0 crashes and both Process 1 and Process 2 miss its message, then both will have
Input B but no Input A (Scenario 3). Since MCTS simulates each of these scenarios independently, it may end up
preferring $[A \rightarrow \texttt{internal:a}]$ in Scenario 2 and preferring $[B \rightarrow \texttt{internal:b}]$
in Scenario 3. And when MCTS combines the probabilities from all scenarios, it may let A and B transit
to different internal states. Again, in this example, MCTS finds a correct solution for each scenario,
but it is incorrect to combine those solutions.

\section{Model Selection and Hyperparameters}
\label{sec:model-parameters}
For model selection, we use a Transformer model, which has become popular in language modeling tasks. The architecture of this model is presented in Table~\ref{table:transformer}. The Transformer model includes only the encoder component of the standard Transformer architecture, as the task does not require contextual information like translation tasks. Additionally, a one-hot encoding layer is used at the input to transform categorical values into unique vectors. The Transformer encoder outputs a tensor of shape {batch $\times$ input\_length $\times$ hidden\_dim}. We apply a global pooling layer to aggregate the outputs across the input sequence into a vector of shape  {batch $\times$ hidden\_dim}, which is then passed through an output layer to obtain the final prediction vector. We also tried the MLP model. Table~\ref{table:mlp} demonstrates the MLP model architecture that we use. It contains three fully connected layers (FC) and one output layer. We use ReLU as the activation function after each FC layer. We use cross-entropy as the loss function for both models.


\begin{table}[t]
\centering
\small
\begin{tabular}{ccc}
\hline
\textbf{Layer Type}     & \textbf{Input Shape} & \textbf{Output Shape} \\ \hline
\textbf{One-hot Encoder}   & (input)      &    \makecell{(input,  encode\_dim)}            \\
\textbf{Transformer Encoder}& \makecell{(input,  encode\_dim)}    &    \makecell{(input, hidden)}             \\ 
\textbf{GlobalAveragePooling1D}           & \makecell{(input, hidden)}     &    (hidden)             \\    
\textbf{Output}& (hidden)      &    (output)               \\
\hline
\end{tabular}
\caption{The Transformer architecture used to learn the state machine}
\label{table:transformer}
\end{table}

\begin{table}[t]
\centering
\small
\begin{tabular}{cccc}
\hline
\textbf{Layer Type}     & \textbf{Size} & \textbf{Output} \\ \hline
\textbf{FC-1}   & 3x128      &    1x128             \\
\textbf{FC-2}& 128x64     &    1x64             \\ 
\textbf{FC-3}           & 64x32      &    1x32              \\    
\textbf{Output}& 32x3      &    1x3               \\
\Xhline{2\arrayrulewidth}
\hline
\end{tabular}
\caption{The MLP architecture used to learn the state machine}
\label{table:mlp}
\end{table}

\section{Experiment Setting}
\label{sec:experiment-setting}

We run all experiments on CloudLab. The server we use is equipped with a 16-core AMD 7302P CPU running at 3.00GHz and has 128GB of memory. We implement the system in Python and use Keras for model training. The learning rate is set to 0.001, and the cross-entropy loss function is used to update the model parameters.
In our experiments, we simulate 100 scenarios in each episode. For each simulation, the number of iterations is determined by the scale of the setting, calculated as $\#\texttt{rounds} \times \#\texttt{processes} \times 1000$. We freeze a new transition every 5 episodes to ensure that all previously frozen transitions have propagated.
To detect training failure, we set a timeout for each run based on the scale. Specifically, we allocate 1 day for 2-process settings, 2 days for 3-process settings, and 4 days for 4-process settings. 
Some distributed protocols require all processes to make the same decision, so a process
may not need the process ID as the input, since processes with the same input, regardless of the process ID, should transit
to the same state. We find that this is true for both the atomic commit
protocol and the consensus protocol we have investigated, so we disabled the process ID in our experiments to
reduce search space.

\section{Competing Approaches}
\label{app:competing-approaches}

We briefly position GGMS among several related families of methods, focusing on differences in goals and outcomes.

\paragraph{Classical distributed protocols and verification.}
Classical distributed protocols such as two-phase commit and atomic commit~\citep{gray1978notes,gray2005notes,Corbett2012Spanner,Zhang2015Tapir}, Paxos and its variants~\citep{lamport2001paxos,Lamport1998Parliament,Lamport1982Byzantine,Moraru2013EPaxos,Ongaro2014Raft,Castro1999PBFT,Kotla2007Zyzzyva,Giridharan2024Autobahn} are hand-designed by experts and then validated by formal proofs or model checking. Frameworks such as IronFleet, I4, DistAI, and Basilisk~\citep{hawblitzel2015ironfleet,ma2019i4,yao2021distai,zhang2025basilisk} start from a human-written protocol and focus on proving its correctness, often for arbitrarily many processes. In contrast, GGMS treats the protocol itself as the object of synthesis: starting only from a specification of safety properties and a fixed number of processes, it searches in the space of global state machines and uses exhaustive model checking to ensure that the learned protocol satisfies the specification. Thus, classical work assumes the protocol and proves properties, whereas GGMS automates protocol construction and then certifies the resulting state machine (for the chosen bound).

\paragraph{Self-play and learning in multi-agent systems.}
Self-play deep reinforcement learning with tree search has achieved superhuman performance in perfect- and imperfect-information games such as backgammon, Atari, Go, poker, and Hanabi~\citep{tesauro1995temporal,mnih2013playing,mnih2015human,silver2017mastering,lample2017playing,wan2018robot,tan2019cooperative,li2020suphx,bard2020hanabi,brown2019superhuman,DBLP:conf/nips/AnthonyTB17,doi:10.1126/science.aay2400,DBLP:conf/iclr/SudhakarNRLRC25}. These methods aim to maximize expected reward and typically output a neural policy that is not formally verified. GGMS borrows Monte Carlo Tree Search and self-play as search and data-generation tools, but with a different objective: instead of winning a game on average, it must synthesize a protocol that satisfies strict safety properties under worst-case crashes and message losses. 

Closer to our setting, \citet{khanchandani2021learning} use self-play to learn algorithms for a distributed directory problem, and verification-guided multi-agent RL approaches (e.g., \cite{Zhang_Li_Chen_Li_Zhang_Zhao_Liu_Chen_2025,Belardinelli2024VerificationMAS}) combine learning with model checking to enforce temporal-logic specifications. These methods learn policies or algorithms and typically use model checking as a verification step or shaping signal, with correctness evaluated on sampled executions or specific scenarios. GGMS instead searches directly in a symbolic space of global state machines for asynchronous crash-prone message-passing systems, and uses model checking as a hard oracle: any violating protocol is rejected, and counterexamples are fed back into training until no counterexamples remain within the explored state space. In this sense, GGMS is closer to protocol synthesis with built-in verification than to performance-driven policy learning.

\paragraph{Automata learning and program synthesis.}
The global state machine representation in GGMS is reminiscent of deterministic finite automata (DFA). Classical DFA learning methods~\citep{Higuera2005Study,Heule2010ExactDFA} infer an automaton consistent with a fixed set of positive and negative example traces. GGMS also produces a finite-state machine, but its data is not a static dataset: executions are generated online via MCTS-guided self-play and iteratively refined using counterexamples from model checking, and the target is satisfaction of distributed safety properties under all crash and message-loss patterns modeled by the verifier.

At a higher level, GGMS is also related to classical program synthesis, particularly counterexample-guided inductive synthesis (CEGIS) as introduced in program sketching~\citep{SolarLezama2008Sketching}. In CEGIS, a synthesizer proposes candidate programs from a constrained search space, and a verifier either accepts the candidate or returns a counterexample input; the counterexample is added to the training set and the loop repeats. GGMS follows a similar high-level architecture: a search procedure proposes candidate protocols and a model checker either accepts them or returns offending executions. However, typical CEGIS systems operate on general-purpose sequential programs (often using SMT-based search) and reason about input/output behavior, whereas GGMS searches over a structured space of distributed protocol state machines under asynchronous semantics and crash/message-loss faults, guided by reinforcement learning and MCTS rather than purely symbolic search. Moreover, GGMS explicitly explores the global state space of the protocol (up to a bound) to guarantee correctness for all executions within that model.

\paragraph{Large-scale reasoning and coding agents.}
Recent large-scale systems also combine powerful search or reinforcement learning with formal environments. AlphaProof is an AlphaZero-inspired agent for formal mathematics that operates inside the Lean theorem prover~\citep{AlphaProof}. Each problem instance is a formal theorem; the agent observes proof states, proposes tactics, and uses MCTS guided by a transformer ``proof network'' to search for proofs, with the Lean kernel checking correctness. The goal is to \emph{find proofs} of given statements within a fixed formal system, and the outcome is a proof term.

AlphaEvolve is an evolutionary coding framework for scientific and algorithmic discovery~\citep{AlphaEvolve}, in which large language models propose edits to candidate programs that are then executed and scored by problem-specific evaluation functions. The framework has been applied to discover new algorithms and improve implementations in several domains. In AlphaEvolve, candidate solutions are arbitrary programs, and their correctness or utility is judged by external evaluation metrics defined by the user.

GGMS is complementary to both: rather than proving theorems in a general-purpose proof assistant or evolving arbitrary code under user-defined metrics, GGMS operates in a fixed asynchronous message-passing model with crash and message-loss faults, searches over a constrained symbolic space of distributed protocols, and uses exhaustive model checking of all executions within this model to certify correctness (for a fixed number of processes).

\section{Experience with LLMs}
\label{sec:GPT}
We tested GPT 5.2 Pro and Gemini 3.0 Pro to see whether they can accomplish a similar goal.
Concretely, we used a prompt to describe the requirements and asked them to generate a protocol;
we manually read the generated protocol provided a counterexample if any; we repeated
this until they get a correct protocol. 

For the consensus protocol, both find a correct version in one shot, by correctly
pointing to the FloodSet algorithm. 

For the atomic commit protocol, our experience is different, probably because we change
the problem definition to some extent: Existing atomic commit protocols like two-phase
commit (2PC) target asynchronous networking environment, but ours targets synchronous
networking environment, and we are not aware of an existing atomic commit protocol targeting
synchronous networking environment. Both GPT and Gemini sometimes can find a correct protocol
in one shot
and sometimes needs multiple iteration with manual verification and counterexamples.
For Gemini, see the one-shot example in \url{https://gemini.google.com/share/6ae02187148b} and
the multiple iteration example in \url{https://gemini.google.com/share/ef4d4300c62f}.
We ran GPT in API mode, stored its output as text files, and submitted them in Supplementary Material
(gpt-atomic-commit1.txt is the one-shot example; gpt-atomic-commit2 is the multiple iteration example;
test.py is our program to call GPT APIs).

Considering that synchronous consensus has known solutions and synchronous atomic commit is still a rather simple protocol, we believe that, though
promising, LLMs still need additional help of verification and exploration to accomplish such tasks. This makes our work complementary to LLMs.
Prior work like AlphaProof and AlphaEvolve~\citep{AlphaProof,AlphaEvolve}
has made the same observation.


\end{appendix}

\end{document}